%% file: rswitch.tex
\documentclass{article}

\usepackage{microtype}
\usepackage{graphicx}
\usepackage{booktabs} 

\usepackage{hyperref}
\usepackage{amsmath}

\usepackage[accepted]{icml2020}

\usepackage{tikz}
\usetikzlibrary{positioning}
\usepackage{subcaption}
\usepackage{amsthm}
\usepackage{thm-restate}
\usepackage{bm}

\newtheorem{theorem}{Theorem}
\newtheorem{cortheorem}{Corollary}[theorem]
\newtheorem{proposition}[theorem]{Proposition}

\newtheorem{defn}{Definition}

\newcommand{\sw}[1]{\widehat{#1}}    
\newcommand{\swx}[2]{\sw{\mathbf{w}}^{#1}_{#2}(\mathbf{x})}
\newcommand{\sbx}[1]{\sw{b}_{#1}(\mathbf{x})}

\icmltitlerunning{Switched linear projections}

\begin{document}

\twocolumn[
\icmltitle{Switched linear projections for neural network interpretability}



\icmlsetsymbol{equal}{*}

\begin{icmlauthorlist}
\icmlauthor{Lech Szymanski}{}
\icmlauthor{Brendan McCane}{}
\icmlauthor{Craig Atkinson}{}
\end{icmlauthorlist}


\icmlcorrespondingauthor{Lech Szymanski}{lechszym@cs.otago.ac.nz}

\icmlkeywords{Artificial Neural Networks, Interpretability, Switched Linear Projections, Input Component Decomposition, Singular Pattern Analysis}

\vskip 0.3in
]




\begin{abstract}
We introduce \textit{switched linear projections} for expressing the activity of a neuron in a deep neural network in terms of a single linear projection in the input space.  The method works by isolating the active subnetwork, a series of linear transformations, that determine the entire computation of the network for a given input instance.  With these projections we can decompose activity in any hidden layer into patterns detected in a given input instance.  We also propose that in ReLU networks it is instructive and meaningful to examine patterns that deactivate the neurons in a hidden layer, something that is implicitly ignored by the existing interpretability methods tracking solely the active aspect of the network's computation.
\end{abstract}

\section{Introduction}


It is notoriously hard to interpret how deep networks accomplish the tasks for which they are trained.  At the same time, due to the pervasiveness of deep learning in numerous aspects of computing, it is increasingly important to gain understanding of how they work.  There are risks associated with the possibility that a neural network might not be ``looking" at the ``right" patterns \cite{Nguyen.etal:2015, Geirhos.etal:2018}, as well as opportunities to learn from a network capable of \textit{better than human} performance \cite{Sadler.etal:2019}.  Hence, there is ongoing effort to improve the interpretation and interpretability of the internal representation of neural networks.   


What makes this interpretation of the inside of a neural network hard is the high dimensionality and the distributed nature of its internal computation.  Aside from the first hidden layer, neurons operate in an abstract high-dimensional space.  If that was not hard enough, the analysis of individual components of the network (such as activity of individual neurons) is rarely instructive, since it is the intricate relationships and interplay of those components that contain the ``secret sauce".  The two broad approaches to dealing with this complexity is to either use simpler interpretable models to approximate what a neural network does, or to trace back the elements of the computation into the input space in order to make the internal dynamics relatable to the input.  In the latter approach we are typically interested in neurons' \textit{sensitivity} -- how the changes in network input affect their output, and \textit{decomposition} -- how different components of the input contribute to the output.  


In this paper we propose a straightforward and elegant method for expressing the computation of an arbitrary neuron's activity to a \textit{single linear projection} in the input space.  This projection consists of a \textit{switched weight vector} and a \textit{switched bias} that easily lend themselves to sensitivity analysis (analogous to gradient-based sensitivity) and decomposition of the internal computation.  We also introduce a new approach to interpretability analysis that disentangles the distributed nature of a hidden layer's representation by decomposition into independent patterns from the input space.  We refer to this method as \textit{singular pattern analysis} (SPA), because it is based on singular value decomposition (SVD) of the matrix of neural activity expressed in terms of switched linear projections.  We also demonstrate that in ReLU networks SPA can be used to separate the representation further into active and inactive parts of the neural network.   

\section{Related work}

Previous work on deep learning interpretability is extensive with a wide variety of methods and approaches -- \cite{Simonyan:etal:2014, Zeiler.etal:2014, Bach.etal:2015, Mahendran.etal:2015, Montavon.etal:2017, Sundararajan.etal:2017, Zhou.etal:2019} being just a selection of the most prominent efforts in this area.  Our work on the single linear projection follows the approach akin to \citet{Lee.etal:2008} and \citet{Erhan.etal2009}, where the objective is to interpret the computation performed by an arbitrary neuron for a particular input vector as a projection in the input space.  However, whereas these previous attempts were based on Deep Belief Nets \cite{Hinton.etal2006} and required an approximation of the said projection, our method is a forward computation that gives the neuron's activity in terms of a linear projection in the input space.  It works for any neural network, including convolutional ones, as long as all hidden neurons use an activation function that is continuous and has a derivative almost everywhere.

Existing methods for interpretability of deep learning representation take it for granted that a given neuron is sensitive to a \textit{feature} and its activity conveys presence of that particular feature in the input.  This approach ignores the fact that the computation inside a neural network is distributed and shared by many neurons (if it was not, dropout would likely not work).  We, on the other hand, take it that neurons in a hidden layer collectively are sensitive to some set of features, but mixed in different combinations for each individual neuron.  We treat interpretability of the internal computation as a blind source separation problem, where the objective is to find independent patterns and the mixing matrix that gives the activity of a layer in a neural network.   

We also take the view that too much \textit{interpretation} in interpretability introduces the risk of showing us what we expect to see and \textit{not} what the network is actually focussing on. For instance, in Deep Taylor Decomposition \cite{Montavon.etal:2017} choices of different root-points for the decomposition of the relevance function lead to different rules for Layerwise Relevance Propagation (LRP)\cite{Bach.etal:2015}, which  can lead to different interpretations of what is important in the input.  The LRP-$\alpha_1\beta_0$ rule, for example, emphasises the computation over the positive weights in the network while discounting the relevance of the information passing through the negative weights.  This rule is justified by assumptions about desired properties of the explanation, but this comes with a risk of confirmation bias.  In SPA there is no interpretation, just descrambling of activity into independent patterns.

\section{Switched linear projections (SLP)}

\input{figure_net.tex}

The basis of SPA is a switched network based on the observation that neurons that produce zero output do not contribute to the computation of the overall output of the network.  The following rationale applies to ReLU networks only.  Later, we will establish how this generalises to networks with other activation functions.

 The notion of \textit{dead} neurons, that is neurons that always output zero, is not new, nor is the realisation that these neurons, along with their connecting weights, can be taken out of a network without any impact on the computation.  In a switched projection, we treat the zero-output neurons as temporarily \textit{dead} for a given instance of input.  We refer to these neurons as \textit{inactive}, since they may become \textit{active} for a different network input.  Thus we isolate the subnetwork of the active neurons in a given computation.  For ReLU activation, the active neurons are those that pass their activity, the weighted sum of their inputs plus bias, directly to their output\footnote{In our terminology, activity denotes output before the activation function and an active neuron is one that produces non-zero output after the activation function; for a ReLU neuron the active and inactive neurons are those that have positive and negative activity respectively.}.  This means that a subnetwork of active ReLU neurons is just a series of linear transformations, which is equivalent to a single linear transformation.  As a result, we can express the computation performed by any neuron in a ReLU network as a projection onto a switched weight vector in the input space plus the switched bias.  The term \textit{switched} indicates that this weight and bias vector changes when the state of the network changes, the state corresponding to the particular combination of the active and inactive neurons in the network.  Figure \ref{fig:switched_relu} illustrates the concept graphically, and a formal description is given in the following theorem:

\begin{restatable}[Switched linear projections for ReLU networks]{theorem}{slpTheorem}
\label{thrm:switchedprojs}
Let $\mathbf{x}\in \mathcal{R}^{d}$ be a  vector of inputs, $\mathbf{w}_{li}\in \mathcal{R}^{U_{l-1}}$ the weight vector, and $b_{li}\in \mathcal{R}$ the bias of neuron $i$ in layer $l$ (with $U_{l-1}$ inputs from the previous layer).  Let the activity of a neuron $i$ in layer $l$ be defined as:
\begin{equation}\label{eqn:activity}
v_{li}(\mathbf{x}) =\Big (\hdots\sigma_r\big (\sigma_r(\mathbf{x}\mathbf{W}_1+\mathbf{b}_1)\mathbf{W}_2+\mathbf{b}_2\big )\hdots\Big )\mathbf{w}_{li}+b_{li},
\end{equation}
where $\mathbf{W}_l=\begin{bmatrix}\mathbf{w}_{l1}^T & \hdots & \mathbf{w}_{lU_l}^T\end{bmatrix}$, $T$ denotes transpose, $\mathbf{b}_l=\begin{bmatrix}b_{l1} & \hdots & b_{lU_l}\end{bmatrix}$ and $\sigma_r(v)=\max(v,0)$ is the ReLU activation function.  If we define an input-dependent state of the network as \\ $\mathbf{W}_l^{(\mathbf{x})} =\begin{bmatrix}\dot{\sigma}_r\big (v_{l1}(\mathbf{x})\big )\mathbf{w}_{l1}^T & \hdots & \dot{\sigma}_r\big (v_{lU_l}(\mathbf{x})\big )\mathbf{w}_{lU_l}^T\end{bmatrix}$ and \\$b_l^{(\mathbf{x})}=\begin{bmatrix}\dot{\sigma}_r(v_{l1}\big (\mathbf{x})\big )b_{l1} & \hdots & \dot{\sigma}_r\big (v_{lU_l}(\mathbf{x})\big )b_{lU_l}\end{bmatrix}$,\\ where $\dot{\sigma}_r(v)=\frac{d\sigma_r(v)}{dv}$, then for\\ $\swx{T}{li}=\mathbf{W}_1^{(\mathbf{x})}\mathbf{W}_2^{(\mathbf{x})}\hdots \mathbf{W}_{l-1}^{(\mathbf{x})}\mathbf{w}^T_{li}$ and  $\sbx{li}=\mathbf{b}_1^{(\mathbf{x})}\mathbf{W}_2^{(\mathbf{x})}\hdots\mathbf{W}_{l-1}^{(\mathbf{x})}\mathbf{w}^T_{li}$ $+\mathbf{b}_2^{(\mathbf{x})}\mathbf{W}_3^{(\mathbf{x})}\hdots\mathbf{W}_{l-1}^{(\mathbf{x})}\mathbf{w}^T_{li}$ $+\hdots+\mathbf{b}_{l-1}\mathbf{w}^T_{li}+b_{li}$, we have

\begin{equation}\label{eqn:swactivity}
v_{li}(\mathbf{x}) =\mathbf{x}\swx{T}{li}+\sbx{li}. 
\end{equation}
\end{restatable}

The proof is provided in Appendix \ref{apx: proof}.  Note that the ReLU derivative, $\dot{\sigma}_r(v)$, is just a convenient definition for a step function, so that 
\begin{equation}
\dot{\sigma}_r(v)\mathbf{w}=\begin{cases}\mathbf{w}, &v>0\\ \mathbf{0}, & \mbox{otherwise.}\end{cases}
\end{equation}
To simplify the notation, whenever referring to the parameters of the switched projection $\sw{\mathbf{w}}$, $\sw{b}$, as well as activity $v$, we will drop the explicit dependency on $\mathbf{x}$.  

While Figure \ref{fig:switched_relu} illustrates the switching concepts on a small fully connected ReLU network, switched linear projections can be computed for networks with convolutional as well as pooling layers.  A convolutional layer is just a special case of a fully connected layer with many weights being zero and groups of neurons constrained to share the weight values on their connections.  For max pooling, the neurons that do not win the competition, and thus their output does not affect the computation from then on, are deemed to be \textit{inactive} regardless of the output they produce.    

In fact, it is fairly obvious from Equation \ref{eqn:swactivity} that a given neuron's switched weight vector is just the derivative of its activity with respect to the network input.  The switched weight vector is just the tangent hyperplane to a given neuron's activity function at $\mathbf{x}$, and the switched bias is the difference $v(x)-\mathbf{x}\sw{\mathbf{w}}^T$.  Thus, we can establish a corollary to Theorem \ref{thrm:switchedprojs} that generalises the notion of the switched linear projection.

\begin{cortheorem}
Switched linear projection for any neuron with activity $v(\mathbf{x})$ where $\mathbf{x}=\begin{bmatrix}x_1 & \hdots & x_d\end{bmatrix}$ is described by the switched weight vector
\begin{equation}
\sw{\mathbf{w}}=\begin{bmatrix}\frac{\partial v(\mathbf{x})}{\partial x_1} & \hdots & \frac{\partial v(\mathbf{x})}{\partial x_d}\end{bmatrix},
\end{equation}
and switched bias
\begin{equation}
\sw{b}=v(\mathbf{x})-\mathbf{x}\sw{\mathbf{w}}^T.
\end{equation}
\end{cortheorem}

Thus libraries with automatic differentiation make the computation of switched linear projection fairly straight forward.

\section{Input component decomposition (ICD)}

The benefit of expressing a neuron's activity in terms of a switched linear projection is that it can be decomposed into contributions from its input.  For this we need additional re-interpretation of activity in order to distribute the contribution of the switched bias over the components of the input vector.  Note that for a linear projection
\begin{equation}\label{eqn:vcentre}
v=\mathbf{x}\mathbf{w}^T+b=(\mathbf{x}-\mathbf{c})\mathbf{w}^T,
\end{equation}
where $\mathbf{c}=\mathbf{x}-\frac{v}{\mathbf{w}\mathbf{w}^T}\mathbf{w}$, $b=-\mathbf{c}\mathbf{w}^T$,  $\mathbf{c}\in\mathcal{R}^d$, and $\mathbf{w} \ne \mathbf{0}$.  Point $\mathbf{c}$ is a point on the hyperplane $\mathbf{x}\mathbf{w}^T=0$ that is closest to $\mathbf{x}$.  The vector $\mathbf{c}$ can be also thought of as a translation of the coordinate system to a neuron-centered one, where $\mathbf{w}$ goes through the origin at $\mathbf{c}$.  \citet{Montavon.etal:2017} call this vector \textit{the nearest root point}, but we will refer to it as the neuron's \textit{centre}.  Since the switched projection is a linear projection we can calculate the switched centre and use it to break down the contribution of the switched projection into $d$ components of the input vector.

\begin{defn}[Input component decomposition (ICD)]
  Given input vector $\mathbf{x}=\begin{bmatrix}x_1 & \hdots x_d\end{bmatrix}$ and a neuron with activity $v=\mathbf{x}\sw{\mathbf{w}}+\sw{b}$, where $\sw{\mathbf{w}}=\begin{bmatrix}\sw{w}_1 & \hdots \sw{w}_d\end{bmatrix}$, $b$ are its switched linear projection, define the ICD vector, $\sw{\bm{\nu}}$ as:
  \begin{equation}
  \sw{\bm{\nu}}=\begin{bmatrix}\sw{\nu}_1 & \hdots & \sw{\nu}_d\end{bmatrix}
  \end{equation}

  where

  \begin{equation}
    \sw{\nu}_j=(x_j-\sw{c}_j)\sw{w}_j
  \end{equation}

  and

  \begin{equation}\label{eqn:switchedc}
    \sw{\mathbf{c}}=\begin{bmatrix}c_1 & \hdots & c_d\end{bmatrix}=\mathbf{x}-\frac{v}{\sw{\mathbf{w}}\sw{\mathbf{w}}^T}\sw{\mathbf{w}}.
  \end{equation}  
\end{defn}

\begin{proposition}
\begin{equation}\label{eqn:vnu}
v=\sum_{j=1}^d \sw{\nu_j},  
\end{equation}
\end{proposition}

\begin{proof}
The proof of the proposition is trivial once Equation \ref{eqn:vnu} is expressed as $v=(\mathbf{x}-\sw{\mathbf{c}})\sw{\mathbf{w}}^T$ and evaluated after the substitution for $\sw{\mathbf{c}}$ from Equation \ref{eqn:switchedc}.  
\end{proof}

The switched centre $\sw{\mathbf{c}}$ is related to the concept of \textit{reference} in DeepLIFT  \cite{Avanti.etal:2017} that gets subtracted from the input in order to extract a pattern of interest.  However, whereas in DeepLIFT the \textit{reference} is external to the model, and used for examination of perturbations induced in the output, our proposed centre is a component of the actual computation of the network's output; one could say, $\sw{\mathbf{c}}$ is a given neuron's inherent \textit{reference}. 

\begin{figure}
  \centering
  \includegraphics[width=0.42\textwidth]{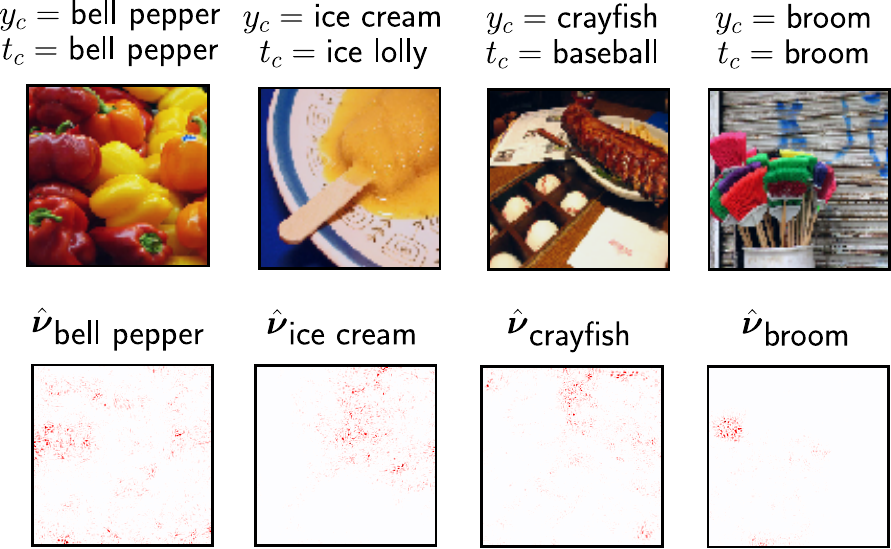}  
  \caption{ICD vectors $\sw{\bm{\nu}}$ of the winning neurons of the VGG16 network for four different input images; each $\sw{\bm{\nu}}$ has been normalised and displayed as a red-blue (negative-positive) heatmap; only red is visible since in each instance the activity of the winning neurons was positive; $y_c$ indicates the predicted label of the winning neuron, $t_c$ the true label.}  
  \label{fig:decomp}
\end{figure}

We take $\nu_j=(x_j-c_j)w_j$ to be the contribution of the input component $j$ to the neuron's activity $v$.  Figure \ref{fig:decomp} shows a visualisation of $\sw{\bm{\nu}}$ of the winning neuron of the VGG16 network \cite{Simonyan.etal:2014} for different input from a subset of images from the Imagenet \cite{imagenet} dataset.  These visualisations do not provide a good picture of what is going on inside the network.  The problem is that input component decomposition, while breaking the activity into contributions of different components, does not account for relationships between those components.  In images, for instance, it is not individual pixels that matter, but patterns across groups of pixels.  We require decomposition into patterns over the components of the input and not contributions from individual components.

\section{Singular pattern analysis (SPA)}

The internal representation of a neural network is distributed over many neurons.  In order to disentangle it, we propose a decomposition of an individual layer's\footnote{Actually, the analysis works for any subset of hidden neurons over the whole network, but it is not clear if it makes sense to decompose patterns across layers.} activity vector $\mathbf{v}$ into set of orthogonal pattern vectors.

Given a set of $M$ neurons in a hidden layer, and their corresponding ICD vectors $\sw{\bm{\nu}}$, we can create a $d\times M$ activity matrix $V$ which summed row-wise gives $\mathbf{v}$.  The SVD of $V$ gives us the desired $d$-dimensional orthogonal vectors.  

\begin{defn}[ICD matrix]
  Let the $d \times M$ ICD matrix of a hidden layer with $M$ neurons be:
  \begin{equation}
    V=\begin{bmatrix}\sw{\bm{\nu}}_1^T & \hdots ... & \sw{\bm{\nu}}_M^T \end{bmatrix}
  \end{equation}
\end{defn}

\begin{defn}[Singular patterns]
  The singular patterns of a layer are the left singular vectors of $V$ as defined by the compact SVD:
\begin{equation}
  V = U \Sigma H.
\end{equation}
where $U = \begin{bmatrix}\mathbf{u}^T_1 & \hdots & \mathbf{u}^T_M\end{bmatrix}$, $\Sigma$ contains the singular values in decreasing order, and $H = \begin{bmatrix}\mathbf{h}^T_1 & \hdots & \mathbf{h}^T_M\end{bmatrix}$.
\end{defn}


Without loss of generality, and for ease of exposition, we assume that $d \ge M$, and hence $U$ is of size $d \times M$, but the analysis is similar when $d<M$.

\begin{figure}
  \centering
  \includegraphics[width=0.38\textwidth]{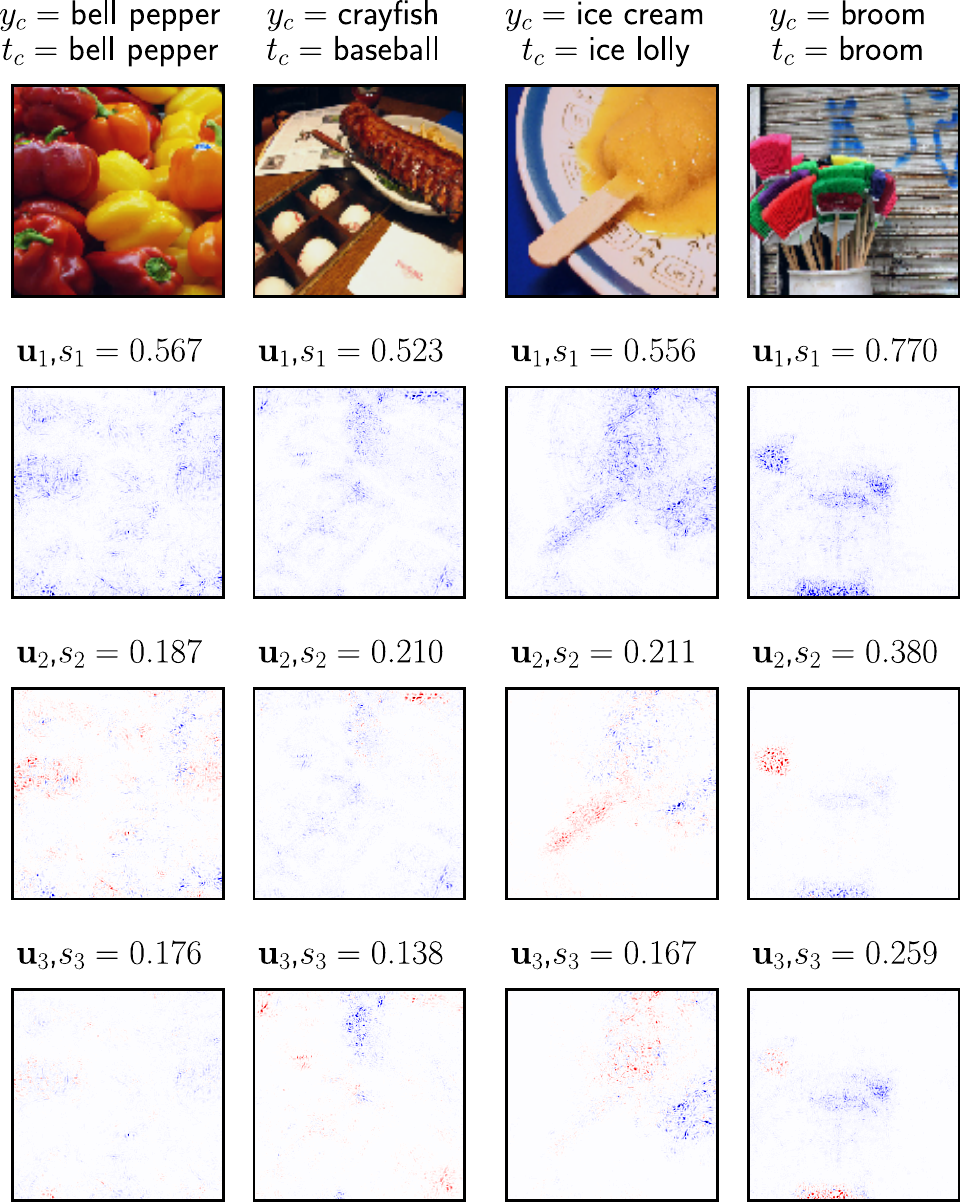}  
  \caption{Three most significant singular patterns of VGG16's output layer's activity for four different inputs; the pattern vectors $\mathbf{u}$ are normalised and shown as blue-red (negative-positive) heatmaps with their corresponding singular values $s$ shown above; significance decreases going from top to bottom; $y_c$ indicates the predicted label of the winning neuron, $t_c$ the true label.}
  \label{fig:ucomp}
\end{figure}

The $M$ vectors $\mathbf{u}_1,...\mathbf{u}_M$ are the patterns in the input space that make up the activity of $\mathbf{v}$.  Singular values relate the scaling of these vectors, and $\mathbf{h}_m$ the mixing coefficients of these patterns to produce activity $v_m$.  

Figure \ref{fig:ucomp} shows visualisations of the three most significant patterns from the output layer of VGG16 for different image inputs.  The significance in this visualisation has been taken from the singular values of the SVD, which weight the power of each pattern in the makeup of the layer's activity.  We'll refer to this weighting as the \textit{broad significance}.  

\begin{figure*}
     \centering
     \begin{subfigure}[b]{0.31\textwidth}
         \centering
         \includegraphics[width=0.8\textwidth]{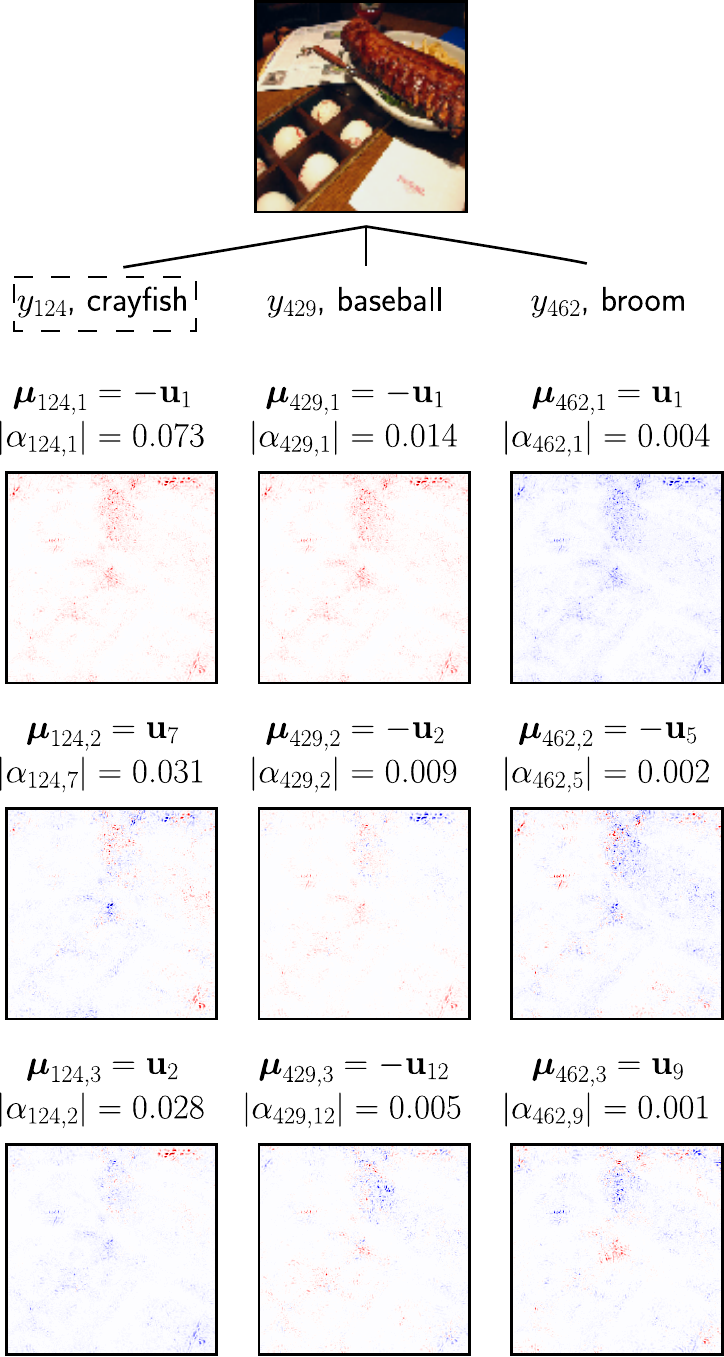}
     \end{subfigure}
     \hfill
     \begin{subfigure}[b]{0.31\textwidth}
         \centering
         \includegraphics[width=0.8\textwidth]{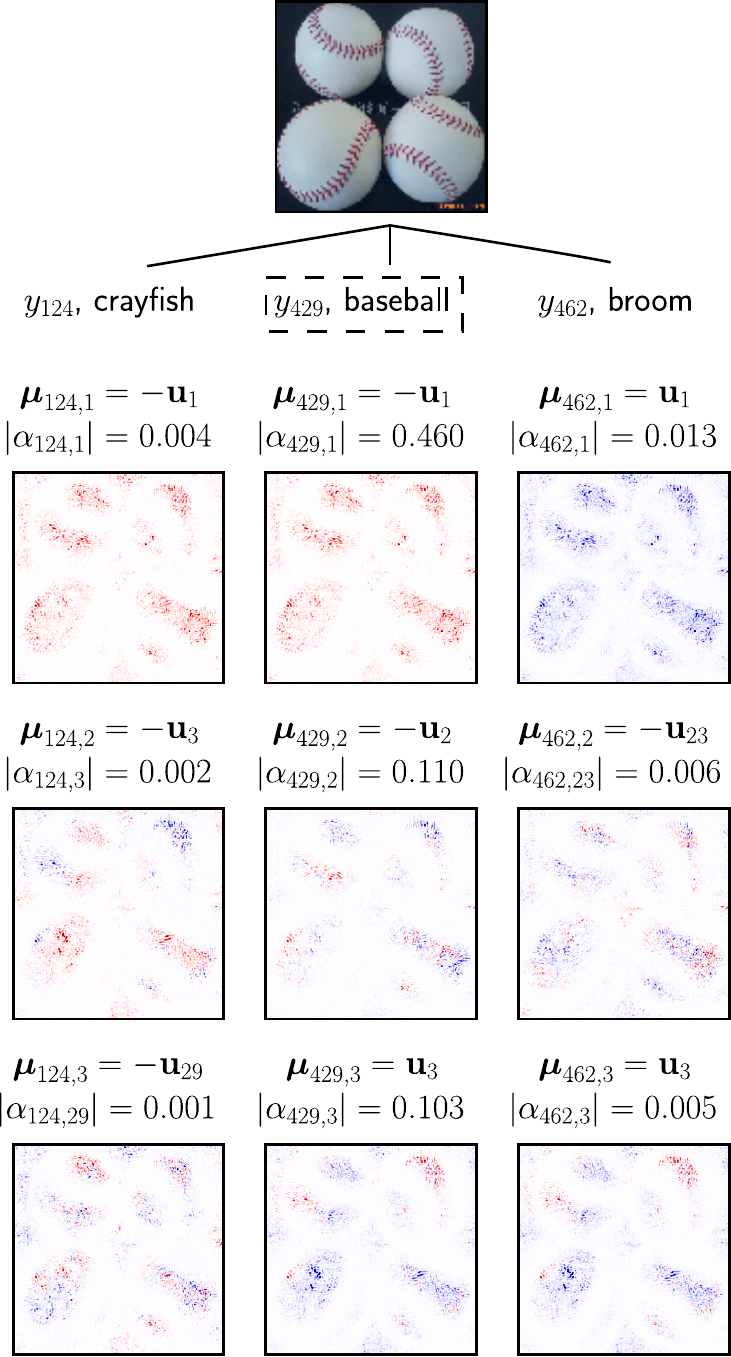}
     \end{subfigure}
     \hfill
     \begin{subfigure}[b]{0.31\textwidth}
         \centering
         \includegraphics[width=0.8\textwidth]{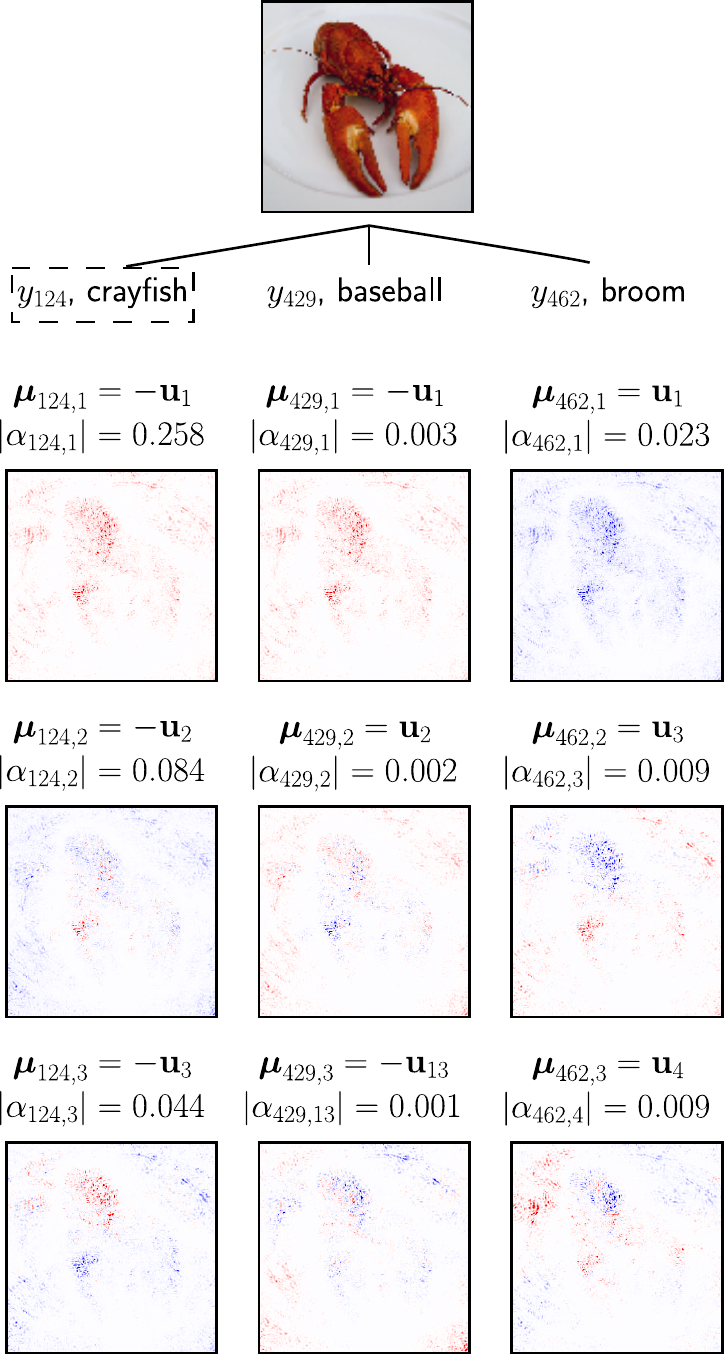}
     \end{subfigure}

        \caption{Three top singular patterns in order of neuron significance for three images: baseballs incorrectly classified by the VGG16 network as ``crayfish" (left), baseballs correctly classified as ``baseball" (middle) and crayfish correctly classified as ``crayfish";  the pattern vectors are normalised and shown as blue-red (negative-positive) heatmaps with labels specifying $\mu_{m,i}=\text{sign}({\alpha_{m,j}})\mathbf{u}_j$ and the absolute value of  $\alpha_{mi}$, where $m$ indexes the output neuron, $i$ the order according to narrow significance, $j$ the order according to broad significance; each column corresponds to different output with its corresponding label; the winning neuron's label is enclosed in a dashed box; neuron significance decreases from top to bottom.}
        \label{fig:vcomp}
\end{figure*}

The mixing coefficients allow us to examine these patterns from the point of view of a single neuron in the layer.  Ordering patterns according to the absolute value of $\alpha_{mi}=s_ih_{mi}$ gives us their \textit{narrow significance}.  Note that the patterns for single-neuron analysis are still the same singular patterns that collectively determine the activity of the entire layer.  The narrow versus broad significance is just different ordering of those patterns.  For instance, if there were two patterns $\mathbf{u}_1$ and $\mathbf{u}_2$ with $s_1=10$ and $s_2=1$ respectively, then their broad significance order tells us that collectively in the activity of all the neurons of that layer $\mathbf{u}_1$ plays a bigger role than $\mathbf{u}_2$.  But if the corresponding mixing coefficients for a particular neuron from that layer were $\alpha_1=0.01$ and $\alpha_2=1.5$, then $s_1\alpha_1 < s_2\alpha_2$ and the narrow significance would reverse the order, meaning there is more $\mathbf{u}_2$ in the activity of that individual neuron than $\mathbf{u}_1$.     

Figure \ref{fig:vcomp} shows three most significant patterns for three (out of thousand) output neurons of the VGG16 network for three different input images.  The patterns are labeled as $\bm{\mu}_{m,i}=\text{sign}(\alpha_{m,j})\mathbf{u}_j$, where $m$ is the index of the neuron in the layer, $i$ is the index of the pattern according to narrow significance ordering, and $j$ is the index of the pattern according to the broad significance ordering.  Since the sign of the mixing coefficient is ignored for the purpose of the narrow ordering, it precedes $\mathbf{u}_j$ and is factored into the visualisations, flipping the colours of the heatmap when $\alpha_{m,j}$ is negative.  The three input images used for the visualisations are: image of baseball (incorrectly classified as ``crayfish") and two correctly classified images -- one of baseballs and the other of crayfish.  

\begin{figure}
  \centering
  \includegraphics[width=0.4\textwidth]{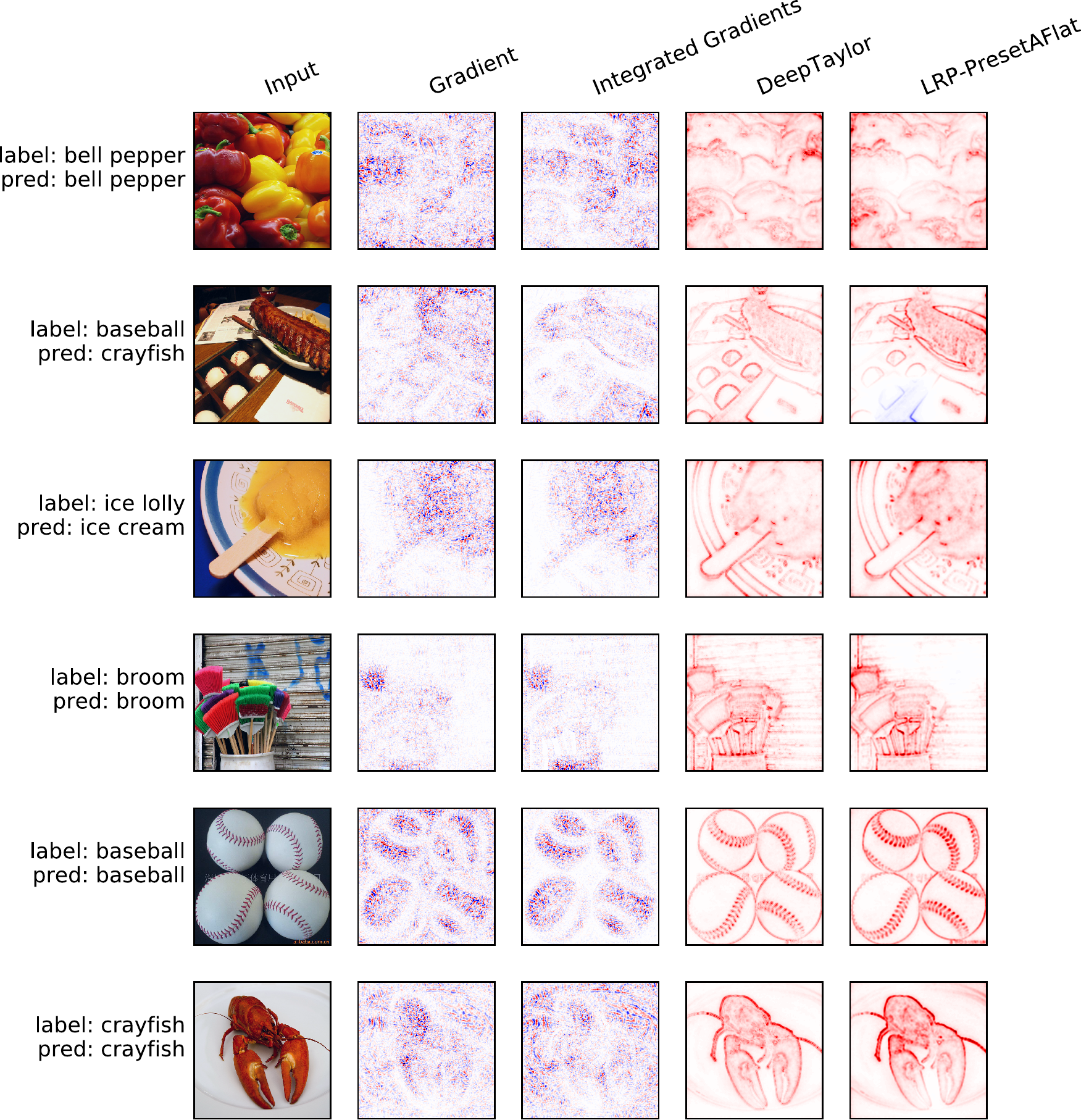}  
  \caption{Visualisations for a set of existing interpretability methods of VGG16 tested on images from the Imagenet dataset.}
  \label{fig:others}
\end{figure}

For comparison, in Figure \ref{fig:others},  we include visualisations of a selection of other interpretability methods: plain gradient sensitivity, Deep Taylor decomposition \cite{Montavon.etal:2017}, Integrated gradients \cite{Sundararajan.etal:2017} and Layerwise relevance propagation \cite{Bach.etal:2015} as implemented by the iNNvestigate toolbox \cite{Alber.etal:2018}.  While they deliver some of the same information as SPA visualisation, they do not break the representation into components as comprehensively as SPA.

\subsection{Representational capacity}

Looking back at Figure \ref{fig:ucomp}, note that the ratio between singular values of the patterns shown varies.  Just like judging the ``true dimensionality" of a set of points through eigenvalues of PCA, we can gauge the utilisation of the representational power of a layer through the singular values of the SPA.  If  all $M$ SPA patterns in an $M$-neuron are equally significant, their singular values will be uniformly distributed corresponding to full use of the entire $M$ dimensions for the corresponding input instance.  If most of the power is concentrated in a fraction of the singular vectors, then only that fraction of representational power is used.  We now have the means of judging the usage of representational capacity of a neural network.     

\begin{defn}[Representational power]
  Given some $0<\gamma<1$, define the instance-based representational power at $\gamma$ for a given layer as:
  \begin{equation}
R_\gamma(x) = | S |,
  \end{equation}
  where $\mathcal{S}=\{\hat{s}|\sum_{i=1}^{|\mathcal{S}|}\hat{s}_i\ge \gamma\}$ and $\hat{s}=s/\sum s_j$ is a normalised singular value from the SVD of $V$ due to some input $\mathbf{x}$. 
\end{defn}

\begin{figure}
  \centering
  \includegraphics[width=0.42\textwidth]{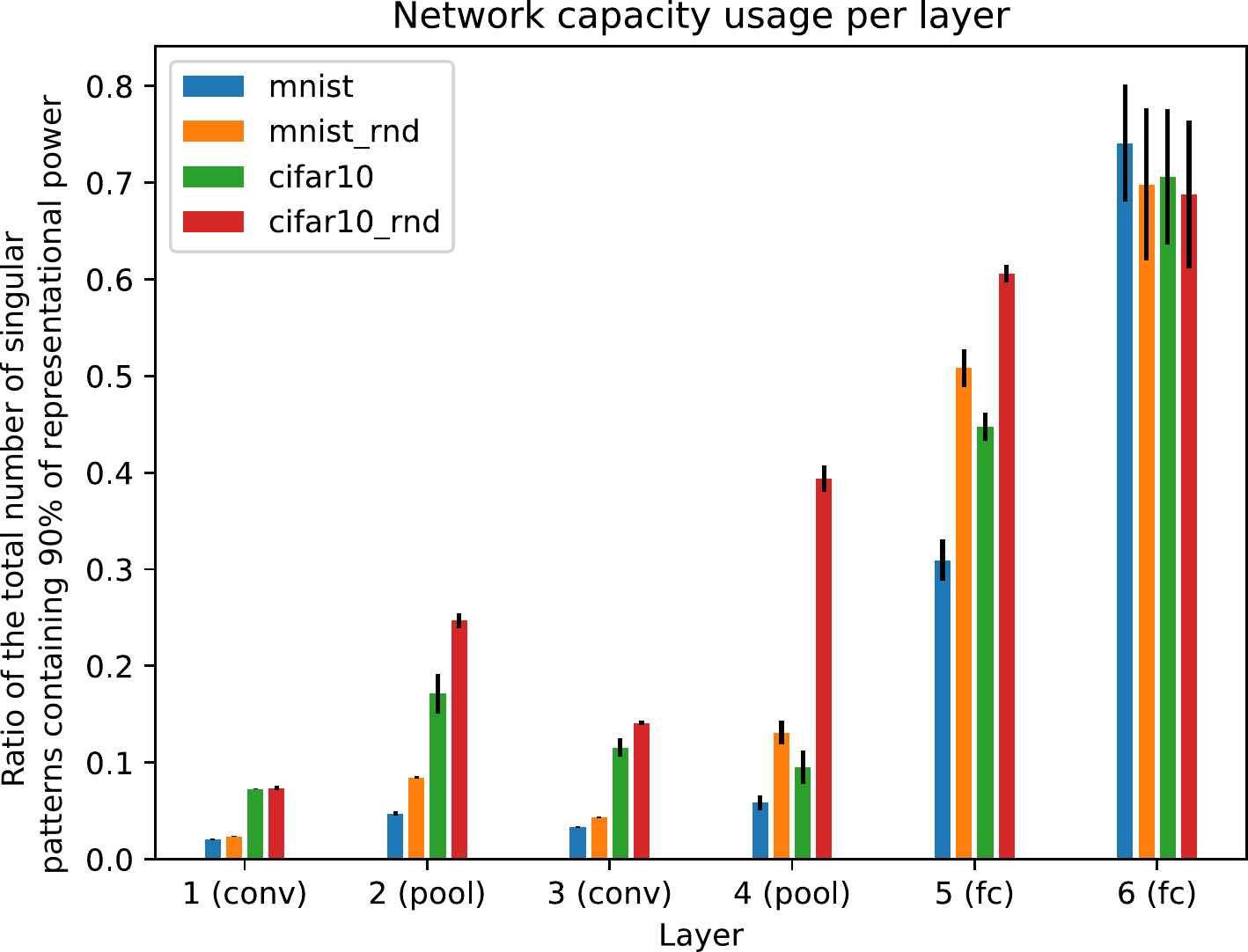}  
  \caption{Instance-based representational capacity usage of layers in a 2CONV neural network  showing the average proportion of the singular values from SPA containing $\gamma=0.9$ of the activity after training on true and randomly labelled MNIST and CIFAR10 datasets; representational power was computed individually for each of 100 randomly chosen images and then averaged.}
  \label{fig:capacity}
\end{figure}

We specify this representational capacity as instance-based, since it only measures capacity of the network usage with respect to individual input instances.  Taking the average over different inputs gives some idea of the overall capacity usage, but it does not tell us whether the same singular patterns recur between input instances, or if the patterns are different.  
 
Figure \ref{fig:capacity} shows the the average instance-based capacity usage of a simple convolutional neural network (from now on referred to as the 2CONV neural network) with two convolutional layers, each followed by a max pool layer, with fully connected penultimate and output layers.  We trained this network on MNIST \cite{mnist} and CIFAR10 \cite{cifar10} datasets -- once with the proper labels (test accuracy of 99.4\% for MNIST and 81\% accuracy on CIFAR10), and a second time with randomised labels (train accuracy\footnote{Test accuracy for random labelled data is low since there is no generalisation for scrambled labels.} of 90\% for random MNIST and 99.8\% for random CIFAR10 respectively).  The experimental setup is similar to that done in \citet{Arpit.etal:2017}, but whereas the authors of that publication had to use various proxies to measure the level of network's ``memorisation", we can simply look at how efficiently the activity of different layers breaks into SPA vectors.  From Figure \ref{fig:capacity} it is evident that MNIST takes less representational power than random MNIST, which is far less than CIFAR10 and random CIFAR10.

The fact that we can examine representational capacity usage layer by layer tells several interesting things about the operation of the network.  In the first layer the capacity usage is the same for a given dataset, regardless of whether it had true or random labelling.  This means that the first layer's representation is related to the complexity of the input space alone, irrespective of the labelling, with MNIST, as expected, being simpler than CIFAR10.  The trend gradually shifts through layers 2 and 3 until in layer 4 and 5 more representational capacity is spent in random labelled as opposed to true labelled datasets; other than that MNIST is still simpler than CIFAR10.  The representation power of the last layer seems to be the same for all datasets with the output just mapping internal computation to the labels.

\section{Inactive state}
\label{sec:inactiveprojs}

\begin{figure}[t]
  \centering
  \includegraphics[width=0.5\textwidth]{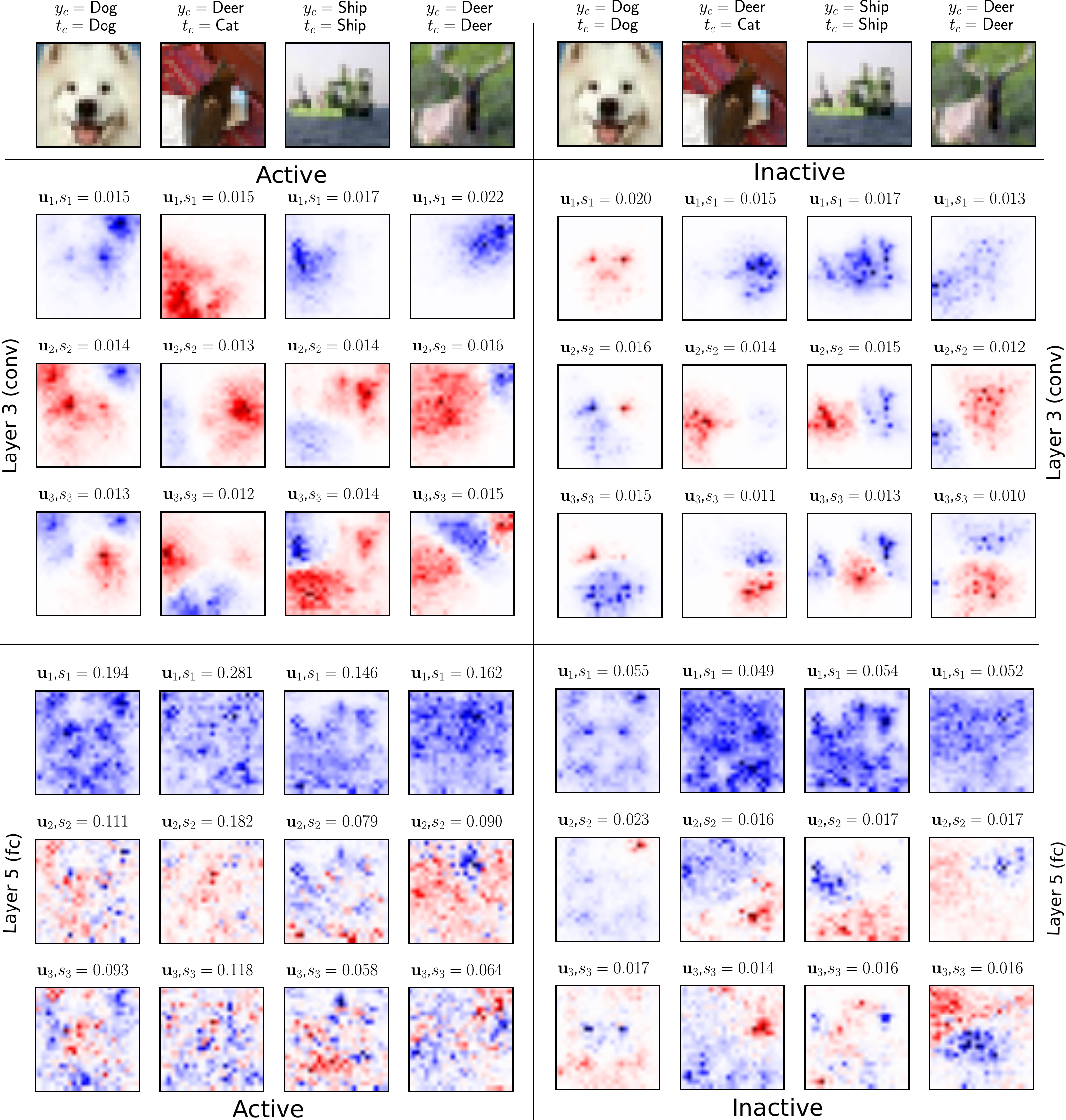}  
  \caption{Singular pattern analysis on the third hidden (convolutional/conv) and fifth hidden (fully connected/fc) layers of the 2CONV architecture trained on CIFAR10 images; each column shows the heatmap visualisations for patterns from the same input image; the four columns on the left show analysis on the active neurons of the corresponding layer, the four columns on the right show analysis on the inactive neurons; the first top row show the input images, following groups of three rows show the three most significant patterns in layer 3 and 5 respectively.} 
  \label{fig:inactive}
\end{figure}

In terms of visualisations, we have been concentrating on the activity of the last layer.  But, as already demonstrated in the previous section, SPA can be used to disentangle the distributed representation in any hidden layer.  In ReLU the decomposition can go even further, as we can do SPA separately on the active and inactive neurons of a hidden layer.  Active neurons are the ones that feed the following layer, hence their representation is of vital importance.  The active network determines the entire computation of the network, yet its makeup is dependent on many neurons being inactive.  We noticed in our experiments, that when operating on Imagenet input, on average 48\% of VGG16 neurons were inactive\footnote{We treat the convolution operations as separate neurons with the same weights connected to different inputs.}.  

When evaluating the test performance of the 2CONV architecture, on average 83\% and 81\% of neurons were inactive after training on the MNIST and CIFAR10 datasets respectively.   The fact that only a subset of neurons are active in a given computation is not a quirk of one specific network, as observed by \citet{Hanin.etal:2019}.    In essence, the particular pattern of activity and inactivity of a ReLU network corresponds to its state that has a bearing on the computation.  Something in the input must ``turn off" the neurons that end up being inactive.  We hypothesise that it is equally meaningful to do SPA on the inactive as it is on the active set of neurons in a neural network.  

Figure \ref{fig:inactive} shows visualisation of the most significant patterns from the SPA of the active and inactive parts of two hidden layers of the 2CONV neural network trained on the CIFAR10 dataset.  Note how the patterns given by the inactive side show fine details of sensitivity to features and regions of the images.

\section{Sanity checks}

\begin{figure}[t]
  \centering
  \includegraphics[width=0.42\textwidth]{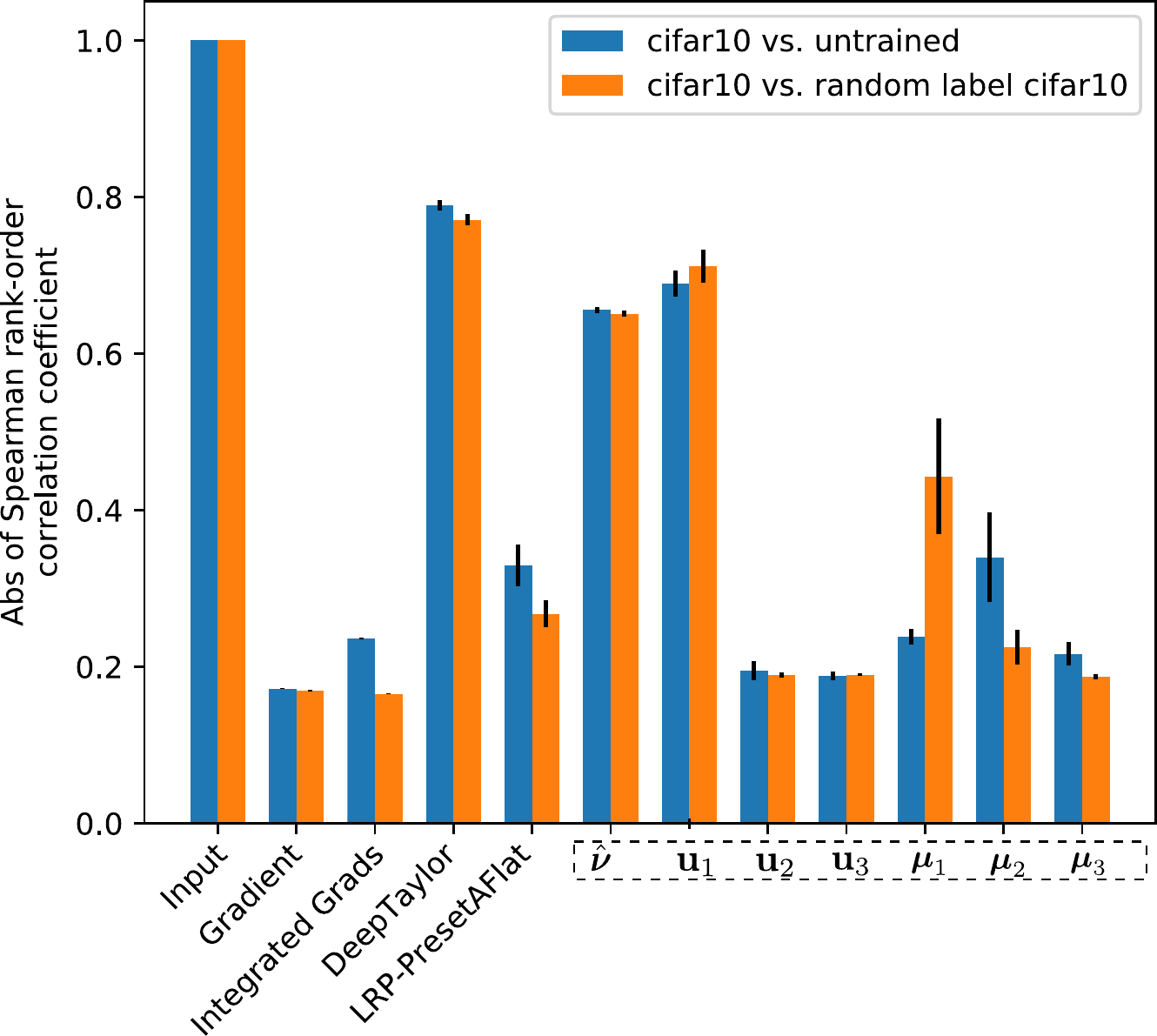}  
  \caption{Mean spearman rank-order correlation between visualisations derived from a random sample of 1000 CIFAR-10 images for different methods between trained and randomly initialised 2CONV network (blue), and between trained and random label CIFAR-10 trained 2CONV (orange); dashed box groups the visualisations obtained from the methods proposed in this paper derived from the activity of the output layer of the network  -- $\sw{\bm{\nu}}$ stands for input component decomposition, $\mathbf{u}_i$ to the $i^{\text{\tiny th}}$ singular pattern and $\bm{\mu}_i$ the $i^{\text{\tiny th}}$ pattern $\text{sign}(\alpha_{mj})$ ranked according to $\alpha_{mj}$.}
  \label{fig:saliency}
\end{figure}

To provide an objective measure of the quality of SLP and SPA based visualisations, we perform sanity and saliency checks as prescribed by \citet{adebayo.etal:2018} on existing and the proposed visualisation methods.    In these tests we measure correlation between interpretability visualisations over the same input for different networks; in the first instance a 2CONV network trained on CIFAR10 against an untrained (randomly initialised) network, and in the second instance against 2CONV trained on randomly labelled CIFAR10.  High correlation between visualisations suggests that an interpretability method is model agnostic, showing similar visualisations regardless of what the network has been trained to detect, or if it was trained at all.  In Figure \ref{fig:saliency} we show average Spearman rank-order correlation coefficients between visualisations generated from 1000 randomly chosen images from the CIFAR10 dataset.  For reference, the correlation between input images is included.  Visualisations proposed in this paper were derived from the output layer of the network and are labeled as $\sw{\bm{\nu}}$ for ICD, $\mathbf{u}_i$ for top $i^{\text{\tiny th}}$ SPA vector according to broad significance and $\bm{\mu}_i$ for top $i^{{\text{\tiny th}}}$ according to narrow significance.     

What is most striking is how badly DeepTaylor visualisation fairs in the sanity checks -- its visualisations are highly correlated regardless of the underlying model.  It is not at all surprising that ICD and the top patterns of SPA, for both broad and narrow significance orderings, are highly correlated between the models.  They relate to the mean of the ICD matrix and this seems to be an echo of the input image.  However, the next two SPA vectors in each ordering show correlations almost as low as gradient-based visualisations, with only $\bm{\mu}_2$ doing a bit worse and on par with LRP.

\section{Conclusion}

The switched linear projection is an interpretation of the computation of activity of a neuron in a network.  It forms the basis of singular pattern analysis which disambiguates the distributed nature of the internal representation inside a neural network.  Activity in a single layer can be decomposed into a set of orthogonal pattern vectors and their corresponding mixing coefficients for each neuron.  Visualisations based on these vectors convey and rank the patterns and the relationships between components of the input in order of their importance for the computation of activity inside and at the output of the network.  It reveals the decision making process inside the network.

SPA delivers the means of gauging the inherent dimensionality of the individual network layer's distributed representation.  This dimensionality can be used to measure the usage of the representational capacity of the network.     

Switch linear projections highlight the difference between the active and inactive components of ReLU neural networks.  Interpretability methods that track the relevance from the output back to the input inadvertently miss the information about the inactive aspect of the network.  We hypothesise that this information is important, since inactive neurons control the non-linear properties of the ReLU network.  SPA can separate and indicate patterns of importance for active and inactive neurons in a hidden layer.  

Since switched linear projections and significant pattern analysis are just an interpretation of the computation inside a neural network, they may also become useful tools for training of deep networks.  For instance, it might be possible to develop new regularisation methods based on switched weights, biases and centres of the neurons in the network, or the inherent dimension of the SPA vectors.  It remains to be investigated how the nature of the inactive subnetwork, and potential ways of manipulating it during training, would affect generalisation. 

\appendix

\section{Proof of Theorem \ref{thrm:switchedprojs}}
\label{apx: proof}


\begin{proof}
By definition from Equation \ref{eqn:activity}, the activity of neuron $i$ in layer $l$ is
\begin{equation}\label{eqn:recursive}
v_{li}(\mathbf{x}) = \sum_{j=1}^{U_l}\sigma_r\Big (v_{l-1j}(\mathbf{x})\Big )w_{lij}+b_{li},
\end{equation}
where $w_{lij}$ is the weight on the connection between neuron $j$ of layer $l-1$ and neuron $i$ of layer $l$, and $b_{li}$ is the bias of neuron $i$ in layer $l$.  

Since
\begin{align*}
  \sigma_r(v)&=\begin{cases}v & v>0\\ 0 & \mbox{otherwise,}\end{cases} \text{ and}\\
  \dot{\sigma}_r(v)&=\frac{d\sigma_r(v)}{dv}=\begin{cases}1 & v>0\\ 0 & \mbox{otherwise},\end{cases}
\end{align*}
we have 
\begin{equation*}
\sigma_r\Big (v_{l-1j}(\mathbf{x})\Big )w_{lij}=\begin{cases}v_{l-1j}(\mathbf{x})w_{lij} & v>0\\ 0 & \mbox{otherwise},\end{cases}
\end{equation*}
and thus
\begin{equation*}
\sigma_r\Big (v_{l-1k}(\mathbf{x})\Big )w_{lik}=v_{l-1k}(\mathbf{x})\dot{\sigma}_r\Big (v_{l-1k}(\mathbf{x})\Big )w_{lik}.
\end{equation*}

As a result
\begin{align}\label{eqn:activeactivity}
  v_{li}(\mathbf{x}) &= \sum_{j=1}^{U_l}v_{l-1j}\dot{\sigma}_r\Big (v_{l-1j}(\mathbf{x})\Big )w_{lij}+b_{li} \nonumber \\
  &=
\sum_{j\in \mathcal{A}}v_{l-1j}w_{lij}+b_{li}
\end{align}

where $\mathcal{A}$ is the set of neurons with activity $v>0$, the active neurons.  Substituting the expression for activity from Equation \ref{eqn:activeactivity} into its recursive definition in Equation \ref{eqn:recursive}, where $v_{0i}(\mathbf{x})=x_i$, reveals that the overall computation is a series of linear transformations of $\mathbf{x}$ equivalent to a single linear transformation 

\begin{equation*}
v_{li}(\mathbf{x}) =\mathbf{x}\swx{T}{li}+\sbx{li}. 
\end{equation*}
where
\noindent
\begin{align*}
\swx{T}{li} = &\mathbf{W}_1^{(\mathbf{x})}\mathbf{W}_2^{(\mathbf{x})}\hdots \mathbf{W}_{l-1}^{(\mathbf{x})}\mathbf{w}^T_{li} \text{ and}\\
  \sbx{li} = &\mathbf{b}_1^{(\mathbf{x})}\mathbf{W}_2^{(\mathbf{x})}\hdots\mathbf{W}_{l-1}^{(\mathbf{x})}\mathbf{w}^T_{li}\\
  &+\mathbf{b}_2^{(\mathbf{x})}\mathbf{W}_3^{(\mathbf{x})}\hdots\mathbf{W}_{l-1}^{(\mathbf{x})}\mathbf{w}^T_{li}\\
  &+\hdots+\mathbf{b}_{l-1}\mathbf{w}^T_{li}+b_{li}.
\end{align*}
\end{proof}

\bibliography{rswitch}
\bibliographystyle{icml2020}


%



\end{document}

%% file: figure_net.tex
\tikzset{%
  every neuron/.style={
    circle,
    draw,
    minimum size=0.8cm,
    fill=gray,
  },
  every output/.style={
    circle,
    draw,
    minimum size=0.8cm,
    fill=white,
    dashed
  },
  every inactive/.style={
    circle,
    draw,
    minimum size=0.8cm,
    fill=white
  },
  every input/.style={
    draw,
    shape=circle,
    scale=0.5,
    fill=black
  },
  every missing/.style={
    draw=none, 
    scale=4,
    text height=0.333cm,
    execute at begin node=\color{black}$\dot$
  },
}

\begin{figure*}
\centering
\begin{subfigure}[b]{0.3\textwidth}
    \centering
    \begin{tikzpicture}[scale=0.7,x=1.5cm, y=1.5cm, >=stealth]
      \node[every input] (x11) at (0,2.5) {};
      \node[every input] (x12) at (0,1.5) {};

      \node[every neuron] (h11) at (1.2,3) {\includegraphics[width=0.45cm]{relu}} ;
      \node[every inactive] (h12) at (1.2,2) {
          \includegraphics[width=0.45cm]{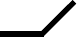}
      };
      \node[every neuron] (h13) at (1.2,1) {\includegraphics[width=0.45cm]{relu.pdf}};

      \node[every inactive] (h21) at (2.6,3) {\includegraphics[width=0.45cm]{relu.pdf}};
      \node[every neuron] (h22) at (2.6,2) {\includegraphics[width=0.45cm]{relu.pdf}};
      \node[every neuron] (h23) at (2.6,1) {\includegraphics[width=0.45cm]{relu.pdf}};

      \node[every output] (h31) at (3.7,2) {$\sum$};

      \draw  [->] (x11) edge (h11)
         (x11) edge (h12)
         (x11) edge (h13)

         (x12) edge (h11)
         (x12) edge (h12)
         (x12) edge (h13)

         (h11) edge (h21)
         (h11) edge (h22)
         (h11) edge (h23)

         (h12) edge (h21)
         (h12) edge (h22)
         (h12) edge (h23)

         (h13) edge (h21)
         (h13) edge (h22)
         (h13) edge (h23)

         (h21) edge (h31)
         (h22) edge (h31)
         (h23) edge (h31)

         ;

      \node[above=0.05cm of x11] {$x_{1}$};

      \node[above=0.05cm of x12]{$x_{2}$};

    \end{tikzpicture}
    \caption{}\label{fig:swa}
  \end{subfigure}
  \hfill
\begin{subfigure}[b]{0.3\textwidth}
    \centering
    \begin{tikzpicture}[scale=0.7,x=1.5cm, y=1.5cm, >=stealth]
      \node[every input] (x11) at (0,2.5) {};
      \node[every input] (x12) at (0,1.5) {};

      \node[every neuron] (h11) at (1.2,3) {\includegraphics[width=0.3cm]{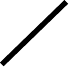}} ;
      \node[every inactive] (h12) at (1.2,2) {\includegraphics[width=0.45cm]{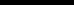}};
      \node[every neuron] (h13) at (1.2,1) {\includegraphics[width=0.3cm]{lin.pdf}};

      \node[every inactive] (h21) at (2.6,3) {\includegraphics[width=0.45cm]{zero.pdf}};
      \node[every neuron] (h22) at (2.6,2) {\includegraphics[width=0.3cm]{lin.pdf}};
      \node[every neuron] (h23) at (2.6,1) {\includegraphics[width=0.3cm]{lin.pdf}};

      \node[every output] (h31) at (3.7,2) {$\sum$};

      \draw  [->] (x11) edge (h11)
         (x11) edge (h12)
         (x11) edge (h13)

         (x12) edge (h11)
         (x12) edge (h12)
         (x12) edge (h13)

         (h11) edge (h21)
         (h11) edge (h22)
         (h11) edge (h23)

         (h12) edge (h21)
         (h12) edge (h22)
         (h12) edge (h23)

         (h13) edge (h21)
         (h13) edge (h22)
         (h13) edge (h23)

         (h21) edge (h31)
         (h22) edge (h31)
         (h23) edge (h31)

         ;

      \node[above=0.05cm of x11] {$x_{1}$};

      \node[above=0.05cm of x12]{$x_{2}$};

    \end{tikzpicture}
    \caption{}
  \end{subfigure}
  \hfill
  \begin{subfigure}[b]{0.3\textwidth}
    \centering
    \begin{tikzpicture}[scale=0.7,x=1.5cm, y=1.5cm, >=stealth]
      \node[every input] (x11) at (0,2.5) {};
      \node[every input] (x12) at (0,1.5) {};

      \node[every neuron] (h11) at (1.2,3) {\includegraphics[width=0.3cm]{lin.pdf}} ;
      \node[every neuron] (h13) at (1.2,1) {\includegraphics[width=0.3cm]{lin.pdf}};

      \node[every neuron] (h22) at (2.6,2) {\includegraphics[width=0.3cm]{lin.pdf}};
      \node[every neuron] (h23) at (2.6,1) {\includegraphics[width=0.3cm]{lin.pdf}};

      \node[every output] (h31) at (3.7,2) {$\sum$};

      \draw  [->] (x11) edge (h11)
         (x11) edge (h13)

         (x12) edge (h11)
         (x12) edge (h13)

         (h11) edge (h22)
         (h11) edge (h23)

         (h13) edge (h22)
         (h13) edge (h23)

         (h22) edge (h31)
         (h23) edge (h31)

         ;

      \node[above=0.05cm of x11] {$x_{1}$};

      \node[above=0.05cm of x12]{$x_{2}$};

    \end{tikzpicture}
    \caption{}
  \end{subfigure}
  \hfill
    \caption{Let's assume that for a particular input $[x_1 \mbox{ } x_2]$ going into the ReLU network shown in (a) the white neurons are \textit{inactive}; then, for this particular input, the network from (a) is equivalent to network in (b) where the inactive neurons are treated as \textit{dead} and the  \textit{active} ones operate in the linear part of their ReLU activation function; which makes both of these networks equivalent to the one in (c); the grey hidden neurons form the active subnetwork.}\label{fig:switched_relu}
\end{figure*}